\documentclass[sigconf]{acmart}

\setcopyright{none}

\usepackage{booktabs} 
\usepackage{bbm}

\usepackage[textsize=tiny]{todonotes}
\DeclareSymbolFont{bbold}{U}{bbold}{m}{n}
\DeclareSymbolFontAlphabet{\mathbbold}{bbold}
\usepackage{enumerate}
\usepackage{dsfont}
\usepackage{algorithm}
\usepackage{algpseudocode}
\usepackage{mathrsfs}
\usepackage{paralist}
\usepackage{epsfig}
\usepackage[caption=false]{subfig}
\usepackage{mathtools}
\usepackage{etoolbox}
\usepackage{pgf}
\usepackage{tikz}
\usetikzlibrary{external}
\usetikzlibrary{automata,topaths,shapes,arrows,fit,decorations.markings}
\tikzstyle{utility}=[diamond,draw=black,draw=blue!50,fill=blue!10,inner sep=0mm, minimum size=8mm]
\tikzstyle{select}=[rectangle,draw=black,draw=blue!50,fill=blue!10,inner sep=0mm, minimum size=6mm]
\tikzstyle{hidden}=[dashed,draw=black]
\tikzstyle{RV}=[circle,draw=black,draw=blue!50,fill=blue!10,inner sep=0mm, minimum size=6mm]

\newcommand \E {\mathop{\mbox{\ensuremath{\mathbb{E}}}}\nolimits}

\renewcommand \Pr {\mathop{\mbox{\ensuremath{\mathbb{P}}}}\nolimits}
\newcommand \PB[1] {\Pr^*(#1)}

\newcommand \defn {\mathrel{\triangleq}}

\newcommand \argmax{\mathop{\rm arg\,max}}
\newcommand \argmin{\mathop{\rm arg\,min}}

\DeclareMathAlphabet{\mathpzc}{OT1}{pzc}{m}{it}

\def\clap#1{\hbox to 0pt{\hss#1\hss}}

\newcommand{\narms}{k}

\newcommand \pol {\pi}

\newcommand \regret {R}
\newcommand \cregret {R}
\newcommand \SDTS {\texttt{SD\_TS}}
\newcommand \FSDTS {\texttt{Fair\_SD\_TS}}

\ifodd 1
\newcommand{\com}[1]{\textbf{\color{red}(COMMENT: #1)}} 
\newcommand{\clar}[1]{\textbf{\color{green}(NEED CLARIFICATION: #1)}}

\newcommand{\response}[1]{\textbf{\color{magenta}(RESPONSE: #1)}} 
\else

\newcommand{\com}[1]{}
\newcommand{\clar}[1]{}
\newcommand{\response}[1]{}
\fi

\acmConference[FAT-ML17]{FAT-ML}{September 2017}{Calibrated Fairness in Bandits} 
\acmYear{2017}
\copyrightyear{2017x}

\title{Calibrated Fairness in Bandits}

\author{Yang Liu}
\affiliation{%
  \institution{SEAS\\Harvard University}
  \city{Cambridge} 
 \state{MA} 
}
\email{yangl@seas.harvard.edu}
\author{Goran Radanovic}
\affiliation{%
  \institution{SEAS\\Harvard University}
  \city{Cambridge} 
 \state{MA} 
}
\email{gradanovic@seas.harvard.edu}
\author{Christos Dimitrakakis}
\affiliation{%
  \institution{Harvard University\\ University of Lille\\Chalmers University of Technology}
}
\email{christos.dimitrakakis@gmail.com}
\author{Debmalya Mandal}
\affiliation{%
  \institution{SEAS\\Harvard University}
  \city{Cambridge} 
 \state{MA} 
}
\email{dmandal@g.harvard.edu}
\author{David C. Parkes}
\affiliation{%
  \institution{SEAS\\Harvard University}
  \city{Cambridge} 
 \state{MA} 
}
\email{parkes@eecs.harvard.edu}

\begin{document}

\begin{abstract}
  We study fairness within the stochastic, \emph{multi-armed bandit}
  (MAB) decision making framework.  We adapt the fairness framework of
  ``treating similar individuals similarly''~\cite{dwork2012fairness}
  to this setting.  Here, an `individual' corresponds to an arm and
  two arms are `similar' if they have a similar quality distribution.
  First, we adopt a {\em smoothness constraint} that if two arms have
  a similar quality distribution then the probability of selecting
  each arm should be similar. In addition, we define the {\em fairness
    regret}, which corresponds to the degree to which an algorithm is
  not calibrated, where perfect calibration requires that the
  probability of selecting an arm is equal to the probability with
  which the arm has the best quality realization.  We show that a
  variation on Thompson sampling satisfies smooth fairness for total
  variation distance, and give an $\tilde{O}((kT)^{2/3})$ bound on
  fairness regret. This complements prior
  work~\cite{joseph2016fairness}, which protects an on-average better
  arm from being less favored.  We also explain how to extend our
  algorithm to the dueling bandit setting.
\end{abstract}

\maketitle

\section{Introduction}
\label{sec:introduction}
Consider a sequential decision making problem where, at each
time-step, a decision maker needs to select one candidate
to hire from a set
of $k$ groups (these may be a different ethnic groups, culture, and
so forth),
whose true qualities are unknown {\em a priori}. The
decision maker would like to make fair decisions with respect to each
group's underlying quality distribution and to learn such a rule
through interactions. This naturally leads to a stochastic \emph{multi-armed
  bandit} framework, where each arm corresponds to a group, and
quality corresponds to reward.

Earlier studies of fairness in bandit problems have emphasized the
need, over all rounds $t$, and for any pair of arms, to weakly favor
an arm that is weakly better in expectation~\cite{joseph2016fairness}.
This notion of meritocratic fairness has provided interesting results,
for example a separation between the dependence on the number of arms
in the regret bound between fair and standard non-fair learning.  But
this is a somewhat weak requirement in that it (i) it allows a group
that is slightly better than all other groups to be selected all the
time and even if any single sample from the group may be worse than
any single sample from another group, and (ii) it allows a random
choice to be made even in the case when one group is much better than
another group.\footnote{\citet{joseph2016rawlsian} also extend the
  results to contextual bandits and infinite bandits.  Here, there is
  additional context associated with an arm in a given time period,
  this context providing information about a specific individual.
  Weak meritocratic fairness requires, for any pair of arms, to weakly
  favor an arm that is weakly better in expectation conditioned on
  context.  When this context removes all uncertainty about quality,
then 
  this extension addresses critique~(i).
But 
in the more general case 
we think it remains interesting for future work to generalize
our definitions to the case of contextual bandits.}

In this work, we adopt the framework of ``treating similar individuals
similarly'' of~\citet{dwork2012fairness}.  In the current context, it
is arms that are the objects about which decisions are made, and thus
the `individual' in Dwork et al.~corresponds to an `arm'. We study the
classic stochastic bandit problem, and insist that over all rounds
$t$, and for any pair of arms, that if the two arms have a similar
quality distribution then the probability with each arm is selected
should be similar.  This {\em smooth fairness} requirement addresses
concern~(i), in that if one group is best in expectation by only a
small margin, but has a similar distribution of rewards to other
groups, then it cannot be selected all the time.

By itself we don't consider smooth fairness to be enough because it 
does not also provide a notion of meritocratic fairness--- it does
not constrain a decision maker in the case that one group is much
stronger than another (in particular, a decision maker could choose
the weaker group). For this reason, we also care about
{\em calibrated fairness} 
and 
introduce the concept of
{\em fairness regret}, which corresponds to the degree to which an
algorithm is not calibrated. Perfect calibration requires that
the probability of selecting a group is equal to the probability that
a group has the best quality realization.  Informally, this is a
strengthening of ``treating similar individuals similarly'' because it
further requires that dissimilar individuals be treated dissimilarly
(and in the right direction.)
In the motivating setting of making decisions about who to hire, 
groups correspond to divisions within society and each activation of
an arm to a particular candidate. An algorithm with low
fairness regret will give individuals a chance proportionally to their
probability of being the best candidate rather than
protect an entire group based on a higher average quality.

\subsection{Our Results}

In regard to smooth fairness, we say that a   bandit algorithm is
  {\em $(\epsilon_1, \epsilon_2,
\delta)$-fair} with respect to a divergence function $D$ (for
$\epsilon_1, \epsilon_2\geq 0$, and $0\leq \delta\leq 1$) if,
with  probability $1-\delta$, 
in every round $t$
and for every pair of arms $i$ and $j$,
$$
D(\pi_t(i)||\pi_t(j))\leq \epsilon_1 D(r_i||r_j)+\epsilon_2,
$$ where $D(\pi_t(i)||\pi_t(j))$ denotes the divergence between the
Bernoulli distributions corresponding to activating arms $i$ and $j$, and $D(r_i||r_j)$ denotes the divergence between the
reward distributions of arms $i$ and $j$.

The  {\em fairness regret} $\cregret_{f,T}$
of a bandit algorithm over $T$ rounds is the total deviation from calibrated fairness:
\begin{equation}
\cregret_{f,T}=\sum_{t=1}^T \E\biggl[ \sum_{ i = 1}^k \max( \PB{i} - \pi_t(i), 0)\biggr]
\end{equation}
where $\PB{i}$ is the probability that the realized quality of arm $i$
is highest and $\pi_t(i)$ is the probability that arm $i$ is activated
by the algorithm in round $t$. 

Our main result is stated for the case of Bernoulli bandits.  We show
that a Thompson-sampling based algorithm, modified to include an
initial uniform exploration phase, satisfies:
\begin{enumerate}
\item $(2,\epsilon_2,\delta)$-fair with regard to total variation
   distance for any $\epsilon_2>0$, $\delta>0$,
  where the amount of initial exploration
  on each arm scales as $1/\epsilon_2^2$
  and $\log(1/\delta)$, and
\item  fairness regret that is bounded by $\tilde{O}((kT)^{2/3})$,
  where $k$ is the number of arms and $T$ the number of rounds.
\end{enumerate}

We also show that a simpler version of Thompson sampling can
immediately satisfy a {\em subjective version} of smooth fairness.
Here, the relevant reward distributions are defined 
with respect to the posterior reward distribution
under the belief of a Bayesian decision maker, this decision
maker having an initially uninformed prior.
In addition, we draw a connection between calibrated fairness and
proper scoring functions: there exists a loss function on reward whose
maximization in expectation would result in a calibrated-fair policy.
In Section~\ref{sec:dueling} we also extend our results to the {\em
  dueling bandit} setting in which the decision maker receives only
pairwise comparisons between arms.

\subsection{Related work}
\label{sec:related-work}

\citet{joseph2016fairness} were the first to introduce fairness
concepts in the bandits setting. These authors adopt the notion of
weak meritocratic fairness, and study it within the classic and
contextual bandit setting. Their main results establish a separation
between the regret for a fair and an un-fair learning algorithm, and
an asymptotically regret-optimal, fair algorithm that uses an approach
of chained confidence intervals.
While their definition promotes meritocracy in regard to expected
quality, this present paper emphasizes instead the distribution on
rewards, and in this way connects with the smoothness definitions and
``similar people be treated similarly'' of~\citet{dwork2012fairness}.

\citet{joseph2016rawlsian} study a more general problem in which there
is no group structure; rather, a number of individuals are available
to select in each period, each with individual context (they also
consider an infinite bandits setting.) \citet{jabbari:fair-mdp} also extend
the notion of weakly meritocratic fairness to Markovian environments,
whereby fairness requires the algorithm to be more likely to play
actions that have a higher utility under the optimal policy.

In the context of fair statistical classification, a number of papers
have asked what does it mean for a method of scoring individuals
(e.g., for the purpose of car insurance, or release on bail) to be
fair. In this setting it is useful to think about each individual as
having a latent outcome, either positive or negative (no car accident,
car accident.) One suggestion is that of {\em statistical parity},
which requires the average score of all members of each group be
equal. For bandits we might interpret the activation probability as
the score, and thus statistical parity would relate to always
selecting each arm with equal probability. Another suggestion is {\em
  calibration within groups}~\cite{kleinberg2016inherent}, which
requires for any score $s\in[0,1]$ and any group, the approximate
fraction of individuals with a positive outcome should be $s$; see
also~\citet{chouldechova2016fair} for a related property.  Considering
also that there is competition between arms in our setting, this
relates to our notion of calibrated fairness, where an arm is
activated according to the probability that its realized reward is
highest. Other definitions first condition on the latent truth; e.g.,
{\em balance}~\cite{kleinberg2016inherent} requires that the expected
score for an individual should be independent of group when
conditioned on a positive outcome; see
also~\citet{DBLP:conf/nips/HardtPNS16} for a related property. These
concepts are harder to interpret in the present context of bandits
problems. Interestingly, these different notions of fair
classification are inherently in conflict with each
other~\cite{kleinberg2016inherent,chouldechova2016fair}.

This statistical learning framework has also been extended to decision
problems by \citet{corbett2017algorithmic}, who analyze the tradeoff
between utility maximization and the satisfaction of fairness
constraints. Another direction is to consider {\em subjective
  fairness}, where the beliefs of the decision maker or external
observer are also taken into
account~\cite{dimitrakakis2017:subjective-fairness}.  The present
paper also briefly considers a specific notion of subjective fairness
for bandits, where the similarity of arms is defined with respect to
their marginal reward distribution.
  
\section{The Setting}
\label{sec:setting}

We consider the stochastic bandits problem, in which at each time
step, a decision maker chooses one of $k$ possible arms (possibly in a
randomized fashion), upon which the decision maker receives a 
reward. We are interested in decision rules that 
are fair in regard to the decisions made
about which arms to activate while achieving high 
total reward.

At each time step $t$, the decision maker chooses a distribution
$\pi_t$ over the available arms, which we refer to as the \textit{decision
  rule}. Then nature draws an action $a_t \sim \pi_t$, and draws
rewards:
\begin{align*} 
r_i(t) | a_t\!=\!i\ \sim P(r_i|\theta_i),  
\end{align*}
where $\theta_i$ is the unknown parameter of the selected arm $a_t = i$,
and where we denote the realized reward for arm $i$ at time $t$ by $r_i(t)$. 

We denote the reward distribution $P(r_i|\theta_i)$ of arm $i$ under
some parameter $\theta_i$ as $r_i(\theta_i)$, with $r_i$ denote the
true reward distribution.  Denote the vector form as
$\mathbf r = (r_1, ..., r_k)$, while $\mathbf r_{-i,j}$ removes $r_i$
and $r_j$ from $\mathbf
r$.
If the decision maker has prior knowledge of the parameters
$\boldsymbol{\theta} = (\theta_1, ..., \theta_k)$, we denote this by
$\beta(\boldsymbol{\theta})$.

\subsection{Smooth Fairness}

For divergence function $D$, let $D(\pi_t(i)\|\pi_t(j))$ to denote the divergence between
the Bernoulli distributions with parameters $\pi_t(i)$ and $\pi_t(j)$,
and use $D(r_i\|r_j)$ as a short-hand for the divergence between the
reward distributions of arm $i$ and $j$ with true parameters
$\theta_i$ and $\theta_j$.

We  define $(\epsilon_1,\epsilon_2,\delta)$-fair w.r.t. a divergence function $D$ for an algorithm with
an associated sequence of decision rules $\{\pi_t\}_t$ as:
\begin{definition}[Smooth fairness]
\label{obj:fair}
A bandit process is $(\epsilon_1,\epsilon_2,\delta)$-fair w.r.t. divergence function $D$, and $\epsilon_1\geq 0,\epsilon_2\geq 0$, $0\leq \delta \leq 1$, 
if with probability at least $1-\delta$, in every round $t$, and for every pair of arms $i$ and $j$:
\begin{align}
D(\pi_t(i)\|\pi_t(j)) \leq \epsilon_1 D(r_i\|r_j) + \epsilon_2.
\end{align}
\end{definition}

\medskip

\paragraph{Interpretation.}  This adapts the concept of ``treating
similar individuals similarly''~\cite{dwork2012fairness} to the
banditrs setting.  If two arms have a similar reward distribution,
then we can only be fair by ensuring that our decision rule has
similar probabilities. The choice of $D$ is crucial. For the KL
divergence, if $r_i, r_j$ do not have common support, our action
distributions may be arbitrarily different. A Wasserstein distance,
requires to treat two arms with a very close mean but different
support similarly to each other. Most of the technical development
will assume the total variation divergence.

As a preliminary, we also consider a variation on smooth fairness
where we would like to be fair with regard to a posterior belief of
the decision maker about the distribution on rewards associated with
each arm.

For this, let the {\em posterior distribution on the parameter}
$\theta_i$ of arm $i$ be $\beta(\theta_i|h^t)$, where
$h^t= (a_1, r_{a_1}(1), \ldots, a_t, r_{a(t)}(t))$, is the history of
observations until time $t$.  The {\em marginal reward distribution
  under the posterior belief} is
\[
r_i(h^t) \defn \int_\Theta P(r_i \mid \theta_i) \, d\beta(\theta_i \mid h^t).
\]
\begin{definition}[Subjective smooth fairness]
A bandit process is $(\epsilon_1,\epsilon_2,\delta)$-subjective fair w.r.t. divergence function $D$, 
and
$\epsilon_1\geq 0, \epsilon_2\geq 0$, and $0\leq \delta\leq 1$,
if,  with probability at least $1-\delta$,
for every period $t$, and every pair of arms $i$ and $j$,  
\begin{align}
D(\pi_t(i)\|\pi_t(j)) \leq \epsilon_1 D(r_i(h^t)\|r_j(h^t)) + \epsilon_2, 
\end{align}
\label{sub:fair}
where the initial belief of the decision maker is an uninformed
prior for each arm.
\end{definition}
\subsection{Calibrated Fairness}

Smooth fairness by itself does not seem strong
enough for fair bandits algorithms. In particular, it does not require
meritocracy: if two arms have quite different reward distributions
then the weaker arm can be selected with higher probability than the
stronger arm. This seems unfair to individuals in the group associated
with the stronger arm.

For this reason we also care about {\em calibrated fairness}: an
algorithm should sample each arm with probability equal to its reward
being the greatest. This would ensure that even very weak arms will be
pulled sometimes, and that better arms will be pulled significantly
more often.
\begin{definition}[Calibrated fair policy]
A policy $\pol_t$ is \emph{calibrated-fair} when it selects actions $a$ with probability
\begin{align}
\pol_t(a) &= \PB{a},
&
\PB{a}
&\defn P(a = \argmax_{j \in [\narms]} \{r_j\}),
\label{eq:calibration}
\end{align}
equal to the probability that the reward realization of arm $a$ is the highest,
and we break ties at random in the case that two arms have the same
realized reward.
\end{definition}

\medskip

Unlike smooth fairness, which can always be achieved exactly (e.g.,
through selecting each arm with equal probability), this notion of
calibrated fairness is not possible to achieve exactly in a bandits
setting while the algorithm is learning the quality of each arm.
For this reason, we  define the cumulative violation of calibration across
all rounds $T$:
\begin{definition}[Fairness regret]
\label{def:fair_regret}
The fairness regret $\regret_f$ of a policy $\pi$ at time $t$ is:
\begin{align*}
\regret_f(t) \defn \E \biggl [ \sum_{ i = 1}^\narms \max( \Pr^*(i) - \pi_t(i), 0)   ~\bigg|~ \boldsymbol{\theta} \biggr].
\end{align*}\
The cumulative fairness regret is defined as $\cregret_{f,T} \defn  \sum_{t = 1}^T \regret_f(t)$.
\end{definition}
\medskip

\begin{example}
  Consider a bandits problem with two arms, whose respective reward
  functions are random variables with realization probabilities:
  \begin{itemize}
  \item $P(r_1 = 1) = 1.0$;
  \item $P(r_2 = 0) = 0.6$ and
    $P(r_2= 2) = 0.4$.
  \end{itemize}

Since $\E(r_1) = 1.0$ and $\E(r_2) = 0.8$, a decision maker who
optimizes  expected payoff (and knows the distributions)
would prefer to always select arm $1$ over arm $2$. Indeed,
this satisfies weakly meritocratic fairness~\cite{joseph2016fairness}.

In contrast, calibrated fairness requires that arm $1$ be selected
60\% of the time and arm $2$ 40\% of the time, since this matches the
frequency with which arm $2$ has the higher realized reward.  In a
learning context, we would not expect an algorithm to be calibrated in
every period. Fairness regret measures the cumulative amount by which
an algorithm is miscalibrated across rounds.

Smooth fairness by itself does not require calibration. Rather, smooth
fairness requires, in every round, that the probability of selecting
arm 1 be close to that of arm 2, where ``close'' depends on the
particular divergence function. In particular, smooth fairness would
not insist on arm 1 being selected with higher probability than arm 2,
without an additional constraint such as maximising expected reward.
\end{example}

\medskip

In Section~\ref{sec:subjective-fairness}, we introduce a simple
Thompson-sampling based algorithm, and show that it satisfies
smooth-subjective fairness. This algorithm provides a building block
towards our main result, which is developed in
Section~\ref{sec:objective}, and provides smooth fairness and low
fairness regret.  Section~\ref{sec:dueling} extends this algorithm to
the dueling bandits setting.

\section{Subjective fairness}
\label{sec:subjective-fairness}

Subjective fairness is a conceptual departure from current approaches
to fair bandits algorithms, which empasize fairness in every period
$t$ with respect to the true reward distributions for each
arm. Rather, subjective fairness adopts the {\em interim} perspective
of a Bayesian decision maker, who is fair with respect to his or her
current beliefs.
Subjective smooth fairness is useful as a building block towards our
main result, which reverts to smooth fairness with regard to the true,
objective reward distribution for each arm.

\subsection{Stochastic-Dominance Thompson sampling}

In \emph{Thompson sampling} (TS), the probability of selecting an arm
is equal to its probability of being the best arm under the subjective
belief (posterior). This draws an immediate parallel with the Rawlsian
notion of equality of opportunity, while taking into account
informational constraints. 

In this section we adopt a simple, multi-level sampling variation,
which we refer to as \emph{stochastic-dominance Thompson sampling},
\SDTS{}.  This first samples parameters $\theta$ from the posterior,
and then samples rewards for each arm, picking the arm with the
highest reward realization.

The version of this algorithm for Bernoulli bandits with a Beta prior,
where each arm's reward is generated according to a Bernoulli random
variable, is detailed in Algorithm \ref{ts:standard}, which considers
the marginal probability of an individual arm's reward realization
being the greatest, and immediately provides subjective smooth
fairness.

\begin{algorithm}[!h]
\caption{(\SDTS{}): Stoch.-Dom.~Thompson sampling}
\begin{algorithmic}
\State For each action $a \in \{1,2,...,\narms\}$, set $S_a=F_a=1/2$ (parameters for priors of Beta distributions). 
\For{$t =1,2,...,$} 
	\State For each action, sample $\theta_a(t)$ from $\text{Beta}(S_a,F_a)$. 
	\State Draw $\tilde{r}_a(t) \sim \text{Bernoulli}(\theta_a(t))$, $\forall a$.
	\State Play arm $a_t :=\text{argmax}_a \tilde{r}_{a}(t)$ (with random tie-breaking). 
	\State Observe the true $r_{a_t}(t)$:
	\begin{itemize}
	\item If $r_{a_t}(t)=1$, $S_{a_t}:=S_{a_t}+1$;
	\item else $F_{a_t} :=F_{a_t}+1$.
	\end{itemize}
\EndFor
\end{algorithmic}
\label{ts:standard}
\end{algorithm}

\begin{theorem}
With (\SDTS{}), we can achieve $(2,0,0)$-subjective fairness under the total variation distance.
 \label{thm:sdts}
\end{theorem}
\begin{proof}
Define:
\begin{align*}
X_{j}(r_i) &= \begin{cases}
1 &\mbox{ if } r_i(h^t)  > \max\{r_j', \mathbf r_{-i,j}'\}\\
0 &\mbox{ if } r_j' > \max\{r_i(h^t) , \mathbf r_{-i,j}'\}\\
Bin(1, \frac{1}{2}) &\mbox{ otherwise }
\end{cases}
\end{align*}
where $r_j' \sim r_j(h^t) $ (similarly for $\mathbf r_{-i,j}'$) and $Bin$ is a binomial random variable. 
First, we have for Thompson sampling:
\begin{align*}
&D(r_i(h^t)\|r_j(h^t))= \frac{1}{2} \cdot D(r_i(h^t)\|r_j(h^t))  + \frac{1}{2} \cdot D(r_i(h^t)\|r_j(h^t)) \\
 &\overset{(a)}{\ge}   \frac{1}{2} \cdot D(X_{i}(r_i(h_t)) \| X_{i}(r_j(h_t)) ) + \frac{1}{2} \cdot D( X_{j}(r_i(h_t)) \| X_{j}(r_j(h_t)))\\
 &= \frac{1}{2} \cdot D (\frac{1}{2} \| \pi_t(j) + \frac{1}{2}\cdot \pi_t(l \ne i,j)) +  \frac{1}{2} \cdot D (\pi_t(i) + \frac{1}{2}\cdot \pi_t(l \ne i,j) \| \frac{1}{2}) \\
&\text{($l$ denotes an arbitrary other agent than $i,j$.)}
\\
 &\overset{(b)}{\ge}  D(\frac{1}{2} \cdot \frac{1}{2} + \frac{1}{2} \cdot \pi_t(i) + \frac{1}{2} \cdot \frac{1}{2}\cdot \pi_t(l \ne i,j)  
 \\
 &\qquad\| \frac{1}{2} \cdot \frac{1}{2} + \frac{1}{2} \cdot \pi_t(j) + \frac{1}{2} \cdot \frac{1}{2}\cdot \pi_t(l \ne i,j) ) 
\\
&= \frac{1}{2} |\pi_t(i) - \pi_t(j)| = \frac{1}{2} \cdot D(\pi_t(i)\|\pi_t(j))
\end{align*}

where step (a) is by monotonicty and step~(b) is by convexity of
divergence function $D$.
Therefore, $\epsilon_1$ is equal to $2$, and $\epsilon_2 = \delta=0$.
\end{proof}

To further reduce  the value of $\epsilon_1$, we can randomize between the selection of the arms in the following manner:
%
\begin{itemize}
\item With probability $\epsilon/2$ select an arm selected by (\SDTS{});
\item Otherwise select uniformly at random another arm.
\end{itemize}

In that case, we have:
\begin{align*}
 &D(\pi_t(i)\|\pi_t(j)) = D(\frac{\epsilon}{2}\pi_{t, ts}(i) + \frac{1-\epsilon}{2} \frac{1}{2}\|  \frac{\epsilon}{2} \pi_{t, ts}(j)  + \frac{1-\epsilon}{2} \frac{1}{2}) \\
 &\overset{\text{monotonicity}}{\le} \frac{\epsilon}{2} D(\pi_{t, ts}(i) \| \pi_{t, ts}(j) ) + \frac{1-\epsilon}{2} D(\frac{1}{2} \| \frac{1}{2}) \\
  &\le \epsilon D(r_i(h^t)\|r_j(h^t)).
\end{align*}

Also see~\citet{sason2016f} for how to to bound $D(r_i(h_t)||r_j(h_t))$ using another $f$-divergence (e.g. through Pinsker's inequality).

\medskip

While, \SDTS{} algorithm is defined in a subjective setting, we can develop a minor variant of it in the objective setting. Even though the original algorithm already uses an uninformative prior,~\footnote{The use of Beta parameters equal to 1/2, corresponds to a Jeffrey's prior for Bernoulli distributions.} to ensure that the algorithm output is more data than prior-driven, in the following section we describe an algorithm, based on \SDTS{},  which can achieve fairness with respect to the actual reward distribution of the arms.


\section{Objective fairness}
\label{sec:objective}

In this section, we introduce a variant of \SDTS{}, which includes an
initial phase of uniform exploration. We then prove the modified
algorithm satisfies (objective) smooth fairness.

Many phased reinforcement learning algorithms~\cite{kearns2002near},
such as those based on successive
elimination~\cite{jmlr:Even-Dar:ActionElimination}, explicitly
separate time into exploration and exploitation phases. In the
exploration phase, arms are prioritized that haven't been selected
enough times.  In the exploitation phase, arms are selected in order
to target the chosen objective as best as possible given the available
information.  The algorithm maintains statistics on the arms, so that
$\mathcal{O}(t)$ is the set which we have not selected sufficiently to
determine their value. Following the structure of the deterministic
exploration algorithm~\cite{vakili2011deterministic}, we exploit
whenever this set is empty, and uniformly choosing among all arms
otherwise.\footnote{However, in our case, the actual drawing of the arms is stochastic to ensure fairness.}

\begin{algorithm}[!h]
  \caption{\FSDTS{}}
  \begin{algorithmic}
    \State At any $t$, denote by $n_i(t)$ the number of times arm $i$ is selected up to time $t$. Check the following set:
    \[
    \mathcal O(t) = \{i: n_i(t) \leq C(\epsilon_2,\delta)\},
    \] 
    where $C(\epsilon_2,\delta)$ depends on $\epsilon_2$ and $\delta$. 
    \begin{itemize}
    \item If $\mathcal O(t) = \emptyset$, follow  (\SDTS{}), using the collected statistics. (\emph{exploitation})
    \item If $\mathcal O(t) \neq \emptyset$, select all arms equally likely. (\emph{exploration})
    \end{itemize}
  \end{algorithmic}
  \label{ts:det}
\end{algorithm}

\begin{theorem}\label{thm:main}
  For any $\epsilon_2,\delta>0$, setting 
  $$
  C(\epsilon_2,\delta) := \frac{(2\max D(r_i||r_j)+1)^2}{2\epsilon^2_2}\log \frac{2}{\delta},
  $$
  we have that (\texttt{Fair\_SD\_TS}) is $(2,2\epsilon_2,\delta)$-fair w.r.t. total variation; and further it has fairness regret bounded as 
  $R_{f,T} \leq \tilde{O}((kT)^{2/3})$.
\end{theorem}

The proof of Theorem~\ref{thm:main} is given in the following sketch.

\begin{proof} (sketch) We begin by proving the first part of Theorem \ref{thm:main}: that for any $\epsilon_2,\delta>0$,
  and  setting 
  $C(\epsilon_2,\delta)$ appropriately, we will have that \texttt{Fair\_SD\_TS} is $(2,2\epsilon_2,\delta)$-fair w.r.t. total variation divergence.
  In the exploration phase,  
  $D(\pi_t(i)||\pi_t(j)) = 0$, so the fairness definition is satisfied. For other steps, using Chernoff bounds we have that with probability at least $1-\delta$
  \[
  |\tilde{\theta}_i - \theta_i| \leq \frac{\epsilon_2}{2\max D(r_i||r_j)+1},\forall i
  \]

  Let the error term for $\theta_i$ be $\epsilon(i)$. 
  Note that for a Bernoulli random variable, we have the following for the mixture distribution:
  \[
  r_i(\tilde{\theta}_i) = (1-\epsilon(i)/2) r(\theta_i) + \epsilon(i)/2 r(1-\theta_i)
  \]
  with $\epsilon(i) \leq \frac{\epsilon_2}{2\max D(r_i||r_j)+1}$. Furthermore, using the convexity of $D$ we can show that:
  \begin{align}
    D(r_i(\tilde{\theta}_i)||r_j(\tilde{\theta}_j)) \leq D(r_i||r_j) + \epsilon_2
  \end{align}

  Following the proof for Theorem~\ref{thm:sdts}, we 
  then obtain that 
  $$D(\pi_t(i)||\pi_t(j)) \leq 2D(r_i(\tilde{\theta}_i)||r_j(\tilde{\theta}_j)),$$ which proves our statement.

  We now establish the fairness regret.  The regret incurred during the
  exploration phase can be bounded as
  $O(\narms^2 C(\epsilon_2,\delta))$.\footnote{This is different from
    standard deterministic bandit algorithms, where the exploration
    regret is often at the order of $\narms C(\epsilon_2,\delta)$. The
    additional $\narms$ factor is due to the uniform selection in the
    exploration phase, while in standard deterministic explorations, the
    arm with the least number of selections will be selected.} For the
  exploitation phase, the regret is bounded by
  $O((\epsilon_2+\delta) T)$.
  Setting 
  \[
  O((\epsilon_2+\delta) T)=O(\narms^2 C(\epsilon_2,\delta))
  \]
  we have the optimal $\epsilon$ is $\epsilon := \narms^{2/3}T^{-1/3}$. Further setting $\delta = O(T^{-1/2})$, we 
  can show the regret is at the order of $ \tilde{O}((\narms T)^{2/3})$.
\end{proof}

It is possible to modify the sampling of the exploitation phase, alternating
between sampling according to \SDTS{} and sampling uniformly
randomly. This can be used  to bring the factor 2 down to any $\epsilon_1 > 0$, at the expense of reduced utility.

\subsection{Connection with proper scoring rules}

There is a connection between calibrated fairness and proper scoring
rules.  Suppose we define a {\em fairness loss function}
$\mathcal L_f$ for decision policy $\mathbf \pi$, such that
$L_{f}(\mathbf \pi) = \mathcal L(\mathbf \pi, a_{t, best})$, where arm
$a_{t, best}$ is the arm with the highest realized reward at time $t$.
The expected loss for policy $\pi$ is
\[
\E (L_f(\mathbf \pi)) = \sum_{i=1}^k \Pr^*(i) \cdot \mathcal L(\mathbf
\pi, i).
\]
If $\mathcal L$ is strictly proper \cite{gneiting2007strictly}, then
the optimal decision rule $\mathbf \pi$ in terms of $L_f$ is
calibrated.
\begin{proposition}
  Consider a fairness loss function $L_f$ defined as:
  \begin{align*}
    L_f(\mathbf \pi) = \mathcal L(\mathbf \pi, a_{t, best}),
  \end{align*}
  where $\mathcal L$ is a strictly proper loss function.
  Then a decision rule $\bar{\mathbf \pi}$ that minimizes expected loss is 
  calibrated fair.
\end{proposition}
\begin{proof}
  We have:
  \begin{align*}
    \bar{\mathbf \pi} \in \argmin_{\mathbf \pi} \E (L_f(\mathbf \pi)) = \argmin_{\mathbf \pi} \sum_{i=1}^k \Pr^*(i) \cdot
    \mathcal L(\mathbf \pi, i) = \{ \Pr^*(i) \},
  \end{align*}
  where the last equality comes from the strict properness of $\mathcal L$.
\end{proof}

This connection between calibration and proper scoring rules 
suggests an approach to the design of bandits algorithms with low
fairness regret, by considering different proper scoring rules along
with online algorithms to minimize loss.


\section{Dueling bandit feedback}
\label{sec:dueling}

After an initial exploration phase, \FSDTS{} selects an arm according
to how likely its sample realization will dominate those of other
arms. This suggests that we are mostly interested in the stochastic
dominance probability, rather than the joint reward distribution.
Recognizing this, we now move to the {\em dueling bandits}
framework~\cite{yue2012k}, which examines pairwise stochastic dominance.

In a
dueling bandit setting, at each time step $t$, the decision maker
chooses two arms $a_t(1),a_t(2)$ to ``duel'' with each other. The
decision maker doesn't observe the actual rewards of $r_{a_t(1)}(t),
r_{a_t(2)}(t)$, but rather the outcome $\mathbbm{1}(r_{a_t(1)}(t) >
r_{a_t(2)}(t)).$ 
In this 
section, we  extend our fairness results to the dueling
bandits setting.

\subsection{A Plackett-Luce model}

Consider the following model.  Denote the probability of arm $i$'s
reward being greater than arm $j$'s reward by:
\begin{align}
  p_{i,j}:=\Pr(i \succ j) := P(r_i > r_j),
\end{align}
where we assume a stationary reward distribution over time $t$.  To be
concrete, we adopt the {\em Plackett-Luce} (PL) model
\cite{guiver2009bayesian,cheng2010label}, where every arm $i$
is parameterized by a {\em quality parameter},
$\nu_i \in \mathbb R_{+}$, such that
\begin{align}
  p_{i,j} = \frac{\nu_i}{\nu_i + \nu_j}. \label{pl:pair}
\end{align} 
Furthermore, let $\mathcal M = [p_{i,j}]$ denote the matrix of
pairwise probabilities $p_{i,j}$. This is a standard setting to
consider in the dueling bandit literature
\cite{yue2012k,szorenyi2015online}.

With knowledge of $\mathcal M$, we can efficiently simulate the best
arm realization. In particular, for the rank over arms
$\texttt{rank} \sim \mathcal M$ generated according to the PL model
(by selecting pairwise comparisons one by one, each time selecting one
from the remaining set with probability proportional to $\nu_i$), we
have~\cite{guiver2009bayesian}:
\[
P(\texttt{rank}|\nu) = \prod_{i=1}^\narms \frac{\nu_{o_i}}{\sum_{j=i}^\narms \nu_{o_j}},
\]
where $\mathbf{o} = \texttt{rank}^{-1}$.
In particular, the marginal probability
in the PL model that an arm is rank 1 (and the best arm) is  just:
\[
P(\texttt{rank}(1) = i) = \frac{\nu_i}{\sum_j \nu_j} = \frac{1}{1+\sum_{j \neq i} \nu_j/\nu_i}.
\]
Finally, knowledge of $\mathcal{M}$ allows us to directly qcalculate each arm's quality from \eqref{pl:pair}.
\[
\nu_j/\nu_i := \frac{p_{j,i}}{1-p_{j,i}}.
\]
Thus, with estimates of the quality parameters ($\mathcal M$) we can estimate $P(\texttt{rank}(1) = i) $ and  directly sample from the best arm distribution and simulate
stochastic-dominance Thompson sampling. 

We will use dueling bandit feedback to estimate pairwise probabilities, denoted
$\tilde{p}_{i,j}$, along with 
the corresponding comparison matrix 
denoted by $\tilde{\mathcal M}$. In particular, let 
$n_{i,j}(t)$  denote the number of times arms $i$ and $j$ are selected up to time $t$. Then
we estimate the pairwise probabilities as:
\begin{align}
  \tilde{p}_{i,j}(t) = \frac{\sum_{n=1}^{n_{i,j(t)}}\mathbbm{1}(r_{i}(n) > r_{j}(n))}{n_{i,j}(t)},~n_{i,j}(t) \geq 1. \label{est:p}
\end{align}

With accurate estimation of the pairwise probabilities, we are able to
accurately approximate the probability that each arm will be rank
1. Denote $\widetilde{\texttt{rank}} \sim \tilde{M}$ as the rank
generated according to the PL model that corresponds with matrix
$\tilde{M}$.  We estimate the ratio of quality parameters
$\widetilde{(\frac{\nu_i}{\nu_j})}$ using $\tilde{p}_{i,j}$, as
\[
\widetilde{(\frac{\nu_i}{\nu_j})} = \frac{\tilde{p}_{i,j} }{1-\tilde{p}_{i,j} }.
\]

Given this, we can then  estimate the probability than arm $i$ has the best
reward realization: 
\begin{align}
  P(\widetilde{\texttt{rank}}(1) = i) = \frac{1}{1+\sum_{j \neq i} \widetilde{(\frac{\nu_j}{\nu_i}})}.\label{est:rank}
\end{align}

%
\begin{lemma}
  When $|\tilde{p}_{i,j} - p_{i,j}| \leq \epsilon$, and $\epsilon$ is small enough, we have that in the Plackett-Luce model,
  $
  |P(\widetilde{\texttt{rank}}(1) = i)  - P(\texttt{rank}(1) = i) | \leq O(\narms \epsilon).
  $\label{rand:perturb}
\end{lemma}

This lemma can be established by establishing a concentration bound on
$\widetilde{(\frac{\nu_i}{\nu_j})}$.  We defer the details to a long
version of this paper.

\begin{algorithm}[!ht]
  \caption{(\texttt{Fair\_SD\_DTS})}
  \begin{algorithmic}
    \State At any $t$, select two arms $a_1(t), a_2(t)$, and receive a realization of the following comparison: $\mathbbm{1}(r_{a_t(1)}(t) > r_{a_t(2)}(t))$.
    \State Check the following set:
    \[
    \mathcal O(t) = \{i: n_{i,j}(t) \leq C(\epsilon_2,\delta)\},
    \]
    where $C(\epsilon_2,\delta)$  depends on $\epsilon_2$ and $\delta$.
    \begin{itemize}
    \item If $\mathcal O(t) = \emptyset$, follow  (\texttt{SD\_TS}), using the collected statistics. (\emph{exploitation})
    \item If $\mathcal O(t) \neq \emptyset$, select all pairs of arms equally likely. (\emph{exploration})
    \end{itemize}
   \State Update $\tilde{p}_{a_t(1),a_t(2)}$, for the pair of selected arms (Eqn. (\ref{est:p})).
   \State Update $P(\texttt{rank}(1) = i)$ using $\tilde{M}$ (Eqn. (\ref{est:rank})).
  \end{algorithmic}
  \label{ts:dueling}
\end{algorithm}

Given this, we can derive an algorithm similar to \FSDTS{} that can
achieve calibrated fairness in this setting, by appropriately setting
the length of the exploration phase, and by simulating the probability
that a given arm has the highest reward realization.  This dueling
version of \FSDTS{} is the algorithm \texttt{Fair\_SD\_DTS}, and
detailed in Algorithm~\ref{ts:dueling}.

Due to the need to explore all pairs of arms, a larger number of
exploration rounds $C(\epsilon_2,\delta)$ is needed, and thus the
fairness regret scales as $R_{f}(T) \leq \tilde{O}(\narms^{4/3}
T^{2/3})$:
\begin{theorem}
  For any $\epsilon_2,\delta>0$, setting 
  $$
  C(\epsilon_2,\delta) \defn O(\frac{(2\max D(r_i||r_j)+1)^2k^2}{2\epsilon^2_2}\log \frac{2}{\delta}),
  $$
  we have that (\texttt{Fair\_SD\_DTS}) is $(2,2\epsilon_2,\delta)$-fair w.r.t. total variation; and further it has fairness regret bounded as 
  $R_{f,T} \leq \tilde{O}(\narms^{4/3}T^{2/3})$.
\end{theorem}

This proof is similar to the fairness regret proof for Theorem \ref{thm:main}, once we established Lemma \ref{rand:perturb}. 
We defer the details to the full version of the paper.


\section{Conclusion}
\label{sec:conclusion}


In this paper we adapt the notion of ``treating similar individuals
similarly''\cite{dwork2012fairness} to the bandits problem, with
similarity based on the distribtion on rewards, and this property of
smooth fairness required to hold along with (approximate) calibrated
fairness.  Calibrated fairness requires that arms that are worse in
expectation still be played if they have a chance of being the best,
and that better arms be played significantly more often than weaker
arms.

We analyzed Thompson-sampling based algorithms, and showed that a
variation with an initial uniform exploration phase can achieve a low
regret bound with regard to calibration as well as smooth fairness.
We further discussed how to adopt this algorithm to a dueling bandit
setting together with Plackett-Luce.

In future work, it will be interesting to consider contextual bandits
(in the case in which the context still leaves residual uncertainty
about quality), to establish lower bounds for fairness regret, to
consider ways to achieve good calibrated fairness uniformly across
rounds, and to study the utility of fair bandits algorithms (e.g.,
with respect to standard notions of regret) and while allowing for a
tradeoff against smooth fairness for different divergence functions
and fairness regret.  In addition, it will be interesting to explore
the connection between strictly-proper scoring rules and calibrated
fairness, as well as to extend Lemma \ref{rand:perturb} to more
general ranking models.

\paragraph{Acknowledgements.}
The research has received funding from: the People Programme (Marie Curie Actions) of the European Union's Seventh Framework Programme (FP7/2007-2013) under REA grant agreement 608743, the Future of Life Institute, and SNSF Early Postdoc.Mobility fellowship.
\bibliographystyle{plainnat}
\bibliography{bibliography,myrefs}

\end{document}